\documentclass[12pt]{article}
\usepackage[a4paper, lmargin=2.9cm, rmargin=2.9cm]{geometry}
\usepackage{amsmath, amsthm, amssymb, dsfont, algorithmic, algorithm, xspace}
\usepackage[breaklinks]{hyperref}
\usepackage[longnamesfirst]{natbib}

\newtheorem{lemma}{Lemma}
\newtheorem{theorem}{Theorem}

\newtheorem{corollary}{Corollary}

\renewenvironment{proof}%
{\begin{trivlist}\item\textbf{Proof.}}%
{\hspace*{\fill}$\Box$\end{trivlist}}

{\begin{trivlist}\item\textbf{Sketch of Proof.}}%
{\hspace*{\fill}$\Box$\end{trivlist}}

\newenvironment{proofof}[1]%
{\begin{trivlist}\item\textbf{Proof of #1.}}%
{\hspace*{\fill}$\Box$\end{trivlist}}

\newcommand{\ea}{(1+1)~EA\xspace}
\newcommand{\oneoneea}{(1+1)~EA\xspace}
\newcommand{\oneoneeaarg}[1]{(1+1)~EA$(#1)$\xspace}
\newcommand{\oneoneeamu}{(1+1)~EA$_\mu$\xspace}

\newcommand{\wrt}{w.\,r.\,t.\xspace}
\newcommand{\wlo}{w.\,l.\,o.\,g.\xspace}
\newcommand{\whp}{w.\,h.\,p.\xspace}

\newcommand{\ie}{i.\,e.\xspace}
\newcommand{\eg}{e.\,g.\xspace}
\newcommand{\ONEMAX}{\textsc{OneMax}\xspace}
\newcommand{\OneMax}{\ONEMAX}
\newcommand{\BV}{\textsc{BinVal}\xspace}

\newcommand{\smin}{s_{\min}}
\newcommand{\cmax}{s_{\max}}
\newcommand{\tp}{\tilde{p}}
\newcommand{\tx}{\tilde{x}}
\newcommand{\ty}{\tilde{y}}

\DeclareMathOperator{\Prob}{Prob}

\DeclareMathOperator*{\argmin}{arg\,min}

\renewcommand{\epsilon}{\varepsilon}
\newcommand{\R}{\mathds{R}}
\newcommand{\N}{\mathds{N}}

\newcommand{\expect}[1]{\mathord{E}\mathord{\left(#1\right)}}

\newcommand{\ones}[1]{\left\lvert #1\right\rvert_1}

\newcommand{\card}[1]{\lvert #1\rvert}
\newcommand{\poly}{\mathrm{poly}}
\newcommand{\prob}[1]{\mathrm{Prob}(#1)}

\newcommand{\BIGOP}[1]
{
\mathop{\mathchoice%
{\raise-0.22em\hbox{\huge $#1$}}%
{\raise-0.05em\hbox{\Large $#1$}}{\hbox{\large $#1$}}{#1}}}
\newcommand{\bigtimes}{\BIGOP{\times}}
\newcommand{\BIGboxplus}{\mathop{\mathchoice%
{\raise-0.35em\hbox{\huge $\boxplus$}}%
{\raise-0.15em\hbox{\Large $\boxplus$}}{\hbox{\large $\boxplus$}}{\boxplus}}}

\title{Tight Bounds on the Optimization Time 
of the \ea on Linear Functions}

\author{Carsten Witt\\\small DTU Informatics, Technical University of Denmark}
\date{\small\today}
\begin{document}
\maketitle

\begin{abstract}
The analysis of randomized search heuristics on classes of functions
is fundamental for the understanding of the underlying stochastic
process and the development of suitable proof techniques. Recently,
remarkable progress has been made in bounding the expected
optimization time of the simple \ea on the class of linear
functions. We improve the best known bound in this setting from
$(1.39+o(1))en\ln n$ to $en\ln n+O(n)$ in expectation and with high
probability, which is tight up to lower-order terms. Moreover, upper
and lower bounds for arbitrary mutations probabilities~$p$ are
derived, which imply expected polynomial optimization time as long as
$p=O((\ln n)/n)$ and which are tight if $p=c/n$ for a constant~$c$. As
a consequence, the standard mutation probability $p=1/n$ is optimal
for all linear functions, and the \oneoneea is found to be an optimal
mutation-based algorithm.  The proofs are based on adaptive drift
functions and the recent multiplicative drift theorem.
\end{abstract}

\section{Introduction}
The rigorous runtime analysis of randomized search heuristics, in
particular of evolutionary computation, is a growing research area
where many results have been obtained in recent years.  This line of
research started off in the early 1990's \citep{Muhlenbein92} with the
consideration of very simple evolutionary algorithms such as the
well-known \ea on very simple example functions such as the well-known
\OneMax function. Later on, results regarding the runtime on classes
of functions were derived \citep[\eg][]{djwea02, heyao2001,
  WegenerW05a, WegenerW05b} and important tools for the analysis were
developed. Nowadays the state of the art in the field allows for the
analysis of different types of search heuristics on problems from
combinatorial optimization \citep{NeumannWittBook}.

Recently, the analysis of evolutionary algorithms on linear
pseudo-boolean functions has experienced a great renaissance.  The
first proof that the \oneoneea optimizes any linear function in
expected time $O(n\log n)$ by \citet{djwea02} was highly technical
since it did not yet explicitly use the analytic framework of drift
analysis \citep{Hajek1982}, which allowed for a considerably
simplified proof of the $O(n\log n)$ bound, see \citet{He2004} for the
first complete proof using the method.\footnote{Note, however, that
  not the original \ea but a variant rejecting offspring of equal
  fitness is studied in that paper.} Another major improvement was
made by \citet{JJ08}, who for the first time stated bounds on the
implicit constant hidden in the $O(n\log n)$ term. This constant was
finally improved by \citet{DJWLinearRevisited} to the bound
$(1.39+o(1))en\ln n$ using a clean framework for the analysis of
multiplicative drift \citep{DJWMultiplicativeDrift}.  The best known
lower bound for general linear functions with non-zero weights is
$en\ln n-O(n)$ and was also proven by \citet{DJWLinearRevisited},
building upon the case of the \OneMax function analyzed by
\citet{DFW10, DFW11}.

The standard \oneoneea flips each bit with probability~$p=1/n$ but
also different values for the mutation probability~$p$ have been
studied in the literature. Recently, it has been proved by
\citet{DoerrGoldbergAdaptiveAlgorithmica} that the $O(n\log n)$ bound on the
expected optimization time of the \oneoneea still holds (also with
high probability) if $p=c/n$ for an arbitrary constant~$c$.  This
result uses the multiplicative drift framework mentioned above and a
drift function being cleverly tailored towards the particular linear
function. However, the analysis is also highly technical and does not
yield explicit constants in the $O$-term. For $p=\omega(1/n)$, no
runtime analyses were known so far.

In this paper, we prove that the \ea optimizes all linear functions in
expected time $en\ln n+O(n)$, thereby closing the gap between the
upper and the lower bound up to terms of lower order. Moreover, we
show a general upper bound depending on the mutation probability~$p$,
which implies that the expected optimization time is polynomial as
long as $p=O((\ln n)/n)$ (and $p=\Omega(1/\poly(n))$). Since the
expected optimization time is proved to be superpolynomial for
$p=\omega((\ln n)/n)$, this implies a phase transition in the regime
$\Theta((\ln n)/n)$. If the mutation probability is $c/n$ for some
constant~$c$, the expected optimization time is proved to be $(1\pm
o(1))\frac{e^c}{c}n\ln n$. Altogether, we obtain that the standard
choice $p=1/n$ of the mutation probability is optimal for all linear
functions. This is remarkable since this seems to be the choice that
is most often recommended by practitioners in evolutionary
computation \citep{Baeck1993}. In fact, the lower bounds hold for 
the large class of
so-called mutation-based EAs, in which the \oneoneea with $p=1/n$ is
found to be an optimal algorithm.

The proofs of the upper bounds use the recent multiplicative drift
theorem and a drift function that is adapted towards both the linear
function and the mutation probability. As a consequence from our main
result, we obtain the results by \citet{DoerrGoldbergAdaptiveAlgorithmica} with less
effort and explicit constants in front of the $n\ln n$-term. All these
bounds hold also with high probability, which follows from the recent
tail bounds added to the multiplicative drift theorem by
\citet{DoerrGoldbergAdaptiveAlgorithmica}. The lower bounds are based on a new
multiplicative drift theorem for lower bounds.
 
This paper is structured as follows.  Section~\ref{sec:prel} sets up
definitions, notations and other
preliminaries. Section~\ref{sec:main-result} summarizes and explains
the main results.  In Sections~\ref{sec:upper} and~\ref{sec:refined},
respectively, we prove an upper bound for general mutation
probabilities and a refined result for $p=1/n$. Lower bounds are shown
in Section~\ref{sec:lower}. We finish with some conclusions.

\section{Preliminaries}
\label{sec:prel}
The \ea is a basic search heuristic for the optimization of
pseudo-boolean functions $f\colon\{0,1\}^n\to\R$. It reflects the
typical behavior of more complicated evolutionary algorithms, serves
as basis for the study of more complex approaches and is therefore
intensively investigated in the theory of randomized search heuristics
\citep{AugerDoerrBook}.  For the case of minimization, it is defined
as Algorithm~\ref{alg:oneoneea}.

\begin{algorithm}
\caption{\oneoneea}
\label{alg:oneoneea}
\begin{algorithmic}
\STATE $t:=0$.
 \STATE choose uniformly at random an initial bit string $x_0 \in \{0,1\}^n$.
 \REPEAT 
  \STATE create $x'$ by flipping each bit in $x_t$ independently with prob.\ $p$ \emph{(mutation)}.
  \STATE $x_{t+1}:=x'$ if $f(x') \le f(x_t)$, and $x_{t+1}:=x_t$ otherwise \emph{(selection)}. 
  \STATE $t:=t+1$.
 \UNTIL{forever.}
\end{algorithmic}
\end{algorithm}

The \oneoneea can be considered a simple hill-climber where search
points are drawn from a stochastic neighborhood based on the mutation
operator.  The parameter $p$, where $0<p<1$, is often chosen as $1/n$,
which then is called \emph{standard mutation probability}.  We call a
mutation from~$x_t$ to~$x'$ \emph{accepted} if $f(x')\le f(x_t)$, \ie, if
the new search point is taken over; otherwise we call it
\emph{rejected.}  In our theoretical studies, we ignore the fact that
the algorithm in practice will be stopped at some time. The
\emph{runtime} (synonymously, \emph{optimization time}) of the \ea is
defined as the first random point in time~$t$ such that the search
point~$x_t$ has optimal, \ie, minimum $f$\nobreakdash-value. This
corresponds to the number of $f$\nobreakdash-evaluations until
reaching the optimum. In many cases, one is aiming for results on the
expected optimization time.  Here, we also prove results that hold
\emph{with high probability (\whp)}, which means probability $1-o(1)$.

The \oneoneea is also an instantiation of the algorithmic scheme that
is called \emph{mutation-based EA} by \citet{SudholtLowerFitnessLevel}
and is displayed as Algorithm~\ref{alg:mutation-based}. It is a
general population-based approach that includes many variants of
evolutionary algorithms with parent and offspring populations as well
as parallel evolutionary algorithms. Any mechanism for managing the
populations, which are multisets, is allowed as long as the mutation
operator is the only variation operator and follows the independent
bit-flip property with probability $0<p\le 1/2$.  Again the
smallest~$t$ such that $x_t$ is optimal defines the runtime. Sudholt
has proved for $p=1/n$ that no mutation-based EA can locate a unique
optimum faster than the \oneoneea can optimize \OneMax. We will see
that the \oneoneea is the best mutation-based EA on a broad class of
functions, also for different mutation probabilities.

\begin{algorithm}
\caption{Scheme of a mutation-based EA}
\label{alg:mutation-based}
\begin{algorithmic}
\FOR{$t:=0 \to \mu-1$}
\STATE create $x_t\in\{0,1\}^n$  uniformly at random.
\ENDFOR
\REPEAT
\STATE select a parent $x \in \{x_0, \dots , x_t\}$ according to $t$ 
 and $f(x_0), \dots , f(x_t)$.
\STATE create $x_{t+1}$ by flipping each bit in $x$ independently with probability $p\le 1/2$.
\STATE $t := t + 1$.
\UNTIL{forever.}
\end{algorithmic}
\end{algorithm}

Throughout this paper, we are concerned with linear pseudo-boolean
functions. A function $f\colon\{0,1\}^n\to\R$ is called \emph{linear}
if it can be written as $f(x_n,\dots,x_1)=w_nx_n+\dots+w_1x_1 +
w_0$. As common in the analysis of the \ea, we assume \wlo\ that
$w_0=0$ and $w_n\ge\dots\ge w_1>0$ hold. Search points are read from
$x_n$ down to~$x_1$ such that $x_n$, the most significant bit, is said
to be on the left-hand side and $x_1$, the least significant bit, on
the right-hand side. Since it fits the proof techniques more
naturally, we assume also \wlo\ that the \ea (or, more generally, the 
mutation-based EA at hand) is minimizing~$f$,
implying that the all-zeros string is the optimum. Our assumptions do
not lose generality since we can permute bits and negate the weights
of a linear function without affecting the stochastic behavior of the
\oneoneea/mutation-based EA.

The probably most intensively studied linear function is
$\OneMax(x_n,\dots,x_1)={x_n+\dots+x_1}$, occasionally also called the
\emph{CountingOnes} problem (which would be the more appropriate name
here since we will be minimizing the function).  In this paper, we
will see that on the one hand, \OneMax is not only the easiest linear
function definition-wise but also in terms of expected optimization
time. On the other hand, the upper bounds obtained for \OneMax hold
for every linear function up to lower-order terms.  Hence,
surprisingly the \oneoneea is basically as efficient on an arbitrary
linear function as it is on \OneMax. This underlines the robustness of
the randomized search heuristic and, in retrospect and for the future,
is a strong motivation to investigate the behavior of randomized
search heuristics on the \OneMax problem thoroughly.

Our proofs of the forthcoming upper bounds use the multiplicative
drift theorem in its most recent version
(cf.\ \citealp{DJWMultiplicativeDrift} and
\citealp{DoerrGoldbergAdaptiveAlgorithmica}).  The key idea of multiplicative
drift is to identify a time-independent relative progress called
drift.

\begin{theorem} [Multiplicative Drift, Upper Bound]
\label{theo:multdrift-uppper}
Let $S\subseteq\R$ be a finite set of positive numbers
with minimum~$1$. Let $\{X^{(t)}\}_{t\ge 0}$ be a 
sequence of random variables over~$S\cup\{0\}$. 
Let~$T$ be the random first point in time~$t\ge 0$ for which $X^{(t)}=0$. 

Suppose that there exists a $\delta>0$ such that
\[
E(X^{(t)}-X^{(t+1)}\mid X^{(t)}= s)\;\ge\; \delta s
\]
for all $s\in S$ with $\Prob(X^{(t)}= s)>0$.
Then for all $s_0\in S$ with $\Prob(X^{(0)}=s_0)>0$,
\[
E(T\mid X^{(0)}=s_0) \;\le\; \frac{\ln(s_0)+1}{\delta}.
\]
Moreover, it holds that 
$\Prob(T> (\ln(s_0)+t)/\delta)) \le e^{-t} $.
\end{theorem}

As an easy example application, consider the \oneoneea on \OneMax and
let $X^{(t)}$ denote the number of one-bits at time~$t$. As worse
search points are not accepted, $X^{(t)}$ is non-increasing over
time. We obtain ${E(X^{(t)}-X^{(t+1)}\mid X^{(t)}= s)} \ge
s(1/n)(1-1/n)^{n-1} \ge s/(en)$, in other words a multiplicative drift
of at least $\delta=1/(en)$, since there are $s$ disjoint single-bit
flips that decrease the $X$-value by~$1$.
Theorem~\ref{theo:multdrift-uppper} applied with $\delta=1/(en)$ and
$\ln(X^{(0)})\le \ln n$ gives us the upper bound $en(\ln n+1)$ on the
expected optimization time, which is the same as the classical method
of fitness-based partitions \citep{WegenerMethods,
  SudholtLowerFitnessLevel} or coupon collector arguments
\citep{Motwani} would yield.

On a general linear function, it is not necessarily a good choice to
let $X^{(t)}$ count the current number of one-bits. Consider, for
example, the well-known function $\BV(x_n,\dots,x_1) = \sum_{i=1}^n
2^{i-1} x_i$. The \oneoneea might replace the search point
$(1,0,\dots,0)$ by the better search point $(0,1,\dots,1)$, amounting
to a loss of $n-2$ zero-bits. More generally, replacing
$(1,0,\dots,0)$ by a better search point is equivalent to flipping the
leftmost one-bit. In such a step, an expected number of $(n-1)p$
zero-bits flip, which decreases the expected number of zero-bits by
only $1-(n-1)p$. The latter expectation (the so-called additive drift)
is only $1/n$ for the standard mutation probability $p=1/n$ and might
be negative for larger~$p$. Therefore, $X^{(t)}$ is typically defined
as $X^{(t)}:=g(x^{(t)})$, where $x^{(t)}$ is the current search point
at time~$t$ and $g(x_n,\dots,x_1)$ is another linear function called
\emph{drift function} or \emph{potential
  function}. \citet{DJWMultiplicativeDrift} use
$x_1+\dots+x_{n/2}+(5/4)(x_{n/2+1}+\dots+x_n)$ as potential function
in their application of the multiplicative drift theorem. This leads
to a good lower bound on the multiplicative drift on the one hand and
a small maximum value of $X^{(t)}$ on the other hand.  In our proofs
of upper bounds in the Sections~\ref{sec:upper} and~\ref{sec:refined},
it is crucial to define appropriate potential functions.

For the lower bounds in Section~\ref{sec:lower}, we need the following
variant of the multiplicative drift theorem.

\begin{theorem}[Multiplicative Drift, Lower Bound]
\label{theo:multdrift-lower}
Let $S\subseteq\mathds{R}$ be a finite set of positive numbers with
minimum $1$.  Let $\{X^{(t)}\}_{t\ge 0}$ be a sequence of random
variables over~$S$, where $X^{(t+1)}\le X^{(t)}$ for any $t\ge 0$, 
and let $\smin > 0$.  Let $T$ be the random first
point in time $t\ge 0$ for which $X^{(t)}\le \smin$. If there exist
positive reals $\beta,\delta\le 1$ such that for all $s>\smin$ and all
$t\ge 0$ with $\Prob(X^{(t)}=s)>0$ it holds that
\begin{enumerate}
\item $\expect{X^{(t)}-X^{(t+1)}\mid X^{(t)}=s} \;\le\; \delta s$, 
\item $\prob{X^{(t)}-X^{(t+1)}\ge \beta s\mid X^{(t)}=s} \;\le\; \beta\delta/\!\ln s$,
\end{enumerate}
then for all $s_0\in S$ with $\Prob(X^{(0)}=s_0)>0$,
\begin{align*}
  \expect{T\mid X^{(0)}=s_0}\;\ge\; \frac{\ln(s_0)-\ln(\smin)}{\delta}\cdot \frac{1-\beta}{1+\beta}.
\end{align*}
\end{theorem}

Compared to the upper bound, the lower-bound version includes a
condition on the maximum stepwise progress and requires 
non-increasing sequences. As a technical detail, the
theorem allows for a positive target $\smin$, which is
required in our applications.

\section{Summary of Main Results}
\label{sec:main-result}

We now list the main consequences from the lower bounds and upper
bounds that we will prove in the following sections.

\begin{theorem}
\label{thm:main-theorem}
On any linear function, the following holds for the 
expected optimization time~$E(T_p)$ of the \oneoneea 
with mutation probability~$p$.
\begin{enumerate}
\item  
If $p=\omega((\ln n)/n)$ or $p=o(1/\poly(n))$ then $E(T_p)$ is superpolynomial.
\item 
If $p=\Omega(1/\poly(n))$ and $p=O((\ln n)/n)$ then $E(T_p)$ is polynomial.
\item 
If $p=c/n$ for 
a constant~$c$ then $E(T_p)=(1\pm o(1))\frac{e^c}{c}n\ln n$.
\item $E(T_p)$ is minimized for mutation probability~$p=1/n$ 
if $n$ is large enough.
\item No mutation-based EA has an expected optimization time 
that is smaller than $E(T_{1/n})$ (up to lower-order terms).
\end{enumerate}
\end{theorem}

In fact, our forthcoming analyses are more precise; in particular, we
do not state available tails on the upper bounds above and leave them
in the more general, but also more complicated
Theorem~\ref{theo:general-upper} in Section~\ref{sec:upper}. The first
statement of our summarizing Theorem~\ref{thm:main-theorem} follows
from the Theorems~\ref{theo:lower-high-p},
\ref{theo:lower-p-larger-half} and \ref{theo:general-lower} in
Section~\ref{sec:lower}.  The second statement is proven in
Corollary~\ref{cor:polynomial-opt}, which follows from the already
mentioned Theorem~\ref{theo:general-upper}. The third statement takes
together the Corollaries~\ref{cor:constant-mut}
and~\ref{cor:constant-mut-lower}. Since $e^c/c$ is minimized for
$c=1$, the fourth statement follows from the third one in conjunction
with Corollary~\ref{cor:constant-mut-lower}. The fifth statement is
also contained in the Theorems~\ref{theo:lower-high-p}
and~\ref{theo:general-lower}.

It is worth noting that the optimality of $p=1/n$ apparently was never
proven rigorously before, not even for the case of \OneMax\footnote{A recent technical report 
extending \citet{SudholtLowerFitnessLevel} shows the optimality of $p=1/n$ in the case of \OneMax using 
a different approach, see \url{http://arxiv.org/abs/1109.1504}.}, where
tight upper and lower bounds on the expected optimization time were
only available for the standard mutation probability
\citep{SudholtLowerFitnessLevel, DFW11}.
For the general case of
linear functions, the strongest previous result said that
$p=\Theta(1/n)$ is optimal \citep{djwea02}. Our result on the
optimality of the mutation probability~$1/n$ is interesting since this
is the commonly recommended choice by practitioners.

\section{Upper Bounds}
\label{sec:upper}

In this section, we show a general upper bound that applies to any
non-trivial mutation probability.

\begin{theorem}
\label{theo:general-upper}
On any linear function, the optimization time of the \ea with mutation
probability $0<p< 1$ is at most
\[
(1-p)^{1-n}\left(\frac{n\alpha^2 (1-p)^{1-n}}{\alpha-1}  + 
\frac{\alpha}{\alpha-1} \frac{\ln(1/p)+(n-1)\ln(1-p)+t}{p}\right) \; =: \; b(t)
\]
with probability at least $1-e^{-t}$, and it is at most $b(1)$ in
expectation, where $\alpha>1$ can be chosen arbitrarily (also
depending on~$n$).
\end{theorem}

Before we prove the theorem, we note two important consequences in
more readable form. The first one (Corollary~\ref{cor:constant-mut})
displays upper bounds for mutation probabilities $c/n$. The second
one (Corollary~\ref{cor:polynomial-opt}) is used in
Theorem~\ref{thm:main-theorem} above, which states a phase transition
from polynomial to superpolynomial expected optimization times at
mutation probability $p=\Theta((\ln n)/n)$.

\begin{corollary}
\label{cor:constant-mut}
On any linear function, the optimization time of the \ea with mutation
probability $p=c/n$, where $c>0$ is a constant, is bounded from above
by $(1+o(1)) ((e^c/c) n\ln n)$ with probability $1-o(1)$ and also in
expectation.
\end{corollary}

\begin{proof}
Let $\alpha:=\ln\ln n$ or any other sufficiently slowly growing
function.  Then $\alpha/(\alpha-1) = 1 + O(1/\!\ln\ln n)$ and
$\alpha^2/(\alpha-1) = O(\ln \ln n)$.  Moreover, $(1-c/n)^{1-n} \le
e^c$. The $b(t)$ in Theorem~\ref{theo:general-upper} becomes at most
\[
e^c\cdot \left(O(n\ln\ln n) + (1+o(1)) \frac{n(\ln(n)+\ln(1/c) + t)}{c}\right),
\]
and the corollary follows by choosing, \eg, $t:=\ln\ln n$.
\end{proof}

\begin{corollary}
\label{cor:polynomial-opt}
On any linear function, the optimization time of the \ea with mutation
probability $p=O((\ln n)/n)$ and $p=\Omega(1/\poly(n))$ is polynomial
with probability $1-o(1)$ and also in expectation.
\end{corollary}

\begin{proof}
Let $\alpha:=2$. By making all positive terms at least~$1$ and
multiplying them, we obtain that the upper bound $b(t)$ from
Theorem~\ref{theo:general-upper} is at most
\[
8n (1-p)^{2-2n} \cdot \frac{\ln(e/p)+t}{p} \;\le\; 8n e^{2pn} \cdot \frac{\ln(e/p)+t}{p}. 
\]
Assume $1/p=\Omega(\poly(n))$ and $p\le c(\ln n)/n$ for some
constant~$c$ and sufficiently large~$n$. Then $e^{2pn} \le n^{2c}$ and
the whole expression is polynomial for $t=1$ (proving the expectation)
and also if $t=\ln n$ (proving the probability $1-o(1)$).
\end{proof}

The proof of Theorem~\ref{theo:general-upper} uses an adaptive
potential function as in \citet{DoerrGoldbergAdaptiveAlgorithmica}. That is, the
random variables $X^{(t)}$ used in Theorem~\ref{theo:multdrift-uppper}
map the current search point of the \oneoneea via a potential function
to some value in a way that depends also on the linear function at
hand. As a special case, if the given linear function happens to be
\OneMax, $X^{(t)}$ just counts the number of one-bits at time~$t$.
The general construction shares some similarities with the one in
\citet{DoerrGoldbergAdaptiveAlgorithmica}, but both construction and proof are less
involved.

\begin{proofof}{Theorem~\ref{theo:general-upper}}
Let $f(x)=w_nx_n+\dots+w_1x_1$ be the linear function at hand. 
Define 
\[
\gamma_i := \left(1+\frac{\alpha p}{(1-p)^{n-1}}\right)^{i-1}
\]
for $1\le i\le n$, and let $g(x)=g_nx_n+\dots+g_1x_1$ be the potential
function defined by $g_1:=1=\gamma_1$ and
\[
g_i \;:=\; \mathord{\min}\mathord{\left\{\gamma_i,\, g_{i-1}\cdot \frac{w_i}{w_{i-1}}\right\}} 
\]
for $2\le i\le n$. Note that the $g_i$ are non-decreasing \wrt~$i$.
Intuitively, if the ratio of $w_i$ and $w_{i-1}$ is too extreme, the
minimum function caps it appropriately, otherwise $g_i$ and $g_{i-1}$
are in the same ratio.  We consider the stochastic process
$X^{(t)}:=g(a^{(t)})$, where $a^{(t)}$ is the current search point of
the \ea at time~$t$. Obviously, $X^{(t)}=0$ if and only if $f$ has
been optimized.

Let $\Delta_t:= X^{(t)}-X^{(t+1)}$. We will show below that
\begin{equation}
E(\Delta_t\mid X^{(t)}= s) \;\ge\;  s\cdot p \cdot (1-p)^{n-1} \cdot \left(1-\frac{1}{\alpha}\right).
\tag{$\ast$}
\label{eq:drift-general-case}
\end{equation}
 The initial value satisfies  
\begin{align*}
X^{(0)}& \;\le\; g_n+\dots+g_1 \;\le\;  \sum_{i=1}^{n} \gamma^{i} 
\;\le\; \frac{\left(1+\frac{\alpha p}{(1-p)^{n-1}}\right)^{n}   -1}{\alpha p(1-p)^{1-n}} \;\le\;
\frac{e^{n\alpha p(1-p)^{1-n}}}{\alpha p (1-p)^{1-n}},
\end{align*}
which means
\[
\ln(X^{(0)}) \;\le\; n\alpha p(1-p)^{1-n} + \ln(1/p) + \ln((1-p)^{n-1}).
\]
The multiplicative drift theorem (Theorem~\ref{theo:multdrift-uppper})
yields that the optimization time~$T$ is bounded from above by
\begin{align*}
  \frac{\ln(X_0)+t}{p (1-p)^{n-1} (1-1/\alpha)} 
 \;\le\; \frac{\alpha\left(n\alpha p(1-p)^{1-n} + \ln (1/p) +  \ln((1-p)^{n-1}) +t\right)}{(\alpha-1)p(1-p)^{n-1}} 
\;=\; b(t)
\end{align*}
with probability at least $1-e^{-t}$, and $E(T)=b(1)$, which proves
the theorem.

To show \eqref{eq:drift-general-case}, we fix an arbitrary current
value~$s$ and an arbitrary search point~$a^{(t)}$ satisfying
$g(a^{(t)})=s$ . In the following, we implicitly assume $X^{(t)}=s$
but mostly omit this for the sake of readability.  We denote by
$I:=\{i\mid a^{(t)}_i=1\}$ the index set of the one-bits in $a^{(t)}$
and by $Z:=\{1,\dots,n\}\setminus I$ the zero-bits.  We assume $I\neq
\emptyset$ since there is nothing to show otherwise. Denote by $a'$
the random (not necessarily accepted) offspring produced by the \ea
when mutating $a^{(t)}$ and by $a^{(t+1)}$ the next search point after
selection. Recall that $a^{(t+1)}=a'$ if and only if $f(a')\le
f(a^{(t)})$. In the following, we will use the event~$A$ that
$a^{(t+1)}=a'\neq a^{(t)}$ since obviously $\Delta_t=0$ otherwise. Let
$I^*:=\{i\in I\mid a'_i=0\}$ be the random set of flipped one-bits and
$Z^*:=\{i\in Z\mid a'_i=1\}$ be the set of flipped zero-bits in~$a'$
(not conditioned on~$A$).  Note that $I^*\neq \emptyset$ if $A$
occurs.

We need further definitions to analyze the drift carefully. For $i\in
I$, we define $k(i):=\max\{j\le i\mid g_j = \gamma_j\}$ as the most
significant position to the right of~$i$ (possibly~$i$ itself) where
the potential function might be capping; note that $k(i)\ge 1$ since
$g_1=\gamma_1$. Let $L(i):=\{k(i),\dots,n\}\cap Z$ be the set of
zero-bits left of (and including) $k(i)$ and let $R(i):=
\{1,\dots,k(i)-1\}\cap Z$ be the remaining zero-bits. Both sets may be
empty.  For event~$A$ to occur, it is necessary that there is
some~$i\in I$ such that bit~$i$ flips to zero and
\[
\sum_{j\in I^*} w_j - \sum_{j\in Z^*\cap L(i)} w_j \;\ge\; 0
\]
since we are taking only zero-bits out of consideration. Now, for
$i\in I$, let $A_i$ be the event that
\begin{enumerate}
\item $i$ is the leftmost flipping one-bit (\ie, $i\in I^*$ and
  $\{i+1,\dots,n\}\cap I^*=\emptyset$) and
\item $\sum_{j\in I^*} w_j - \sum_{j\in Z^*\cap L(i)} w_j \;\ge\; 0$.
\end{enumerate}
If none of the $A_i$ occurs, $\Delta_t=0$. Furthermore, the $A_i$ are
mutually disjoint.

For any $i\in I$, $\Delta_t$ can be written as the sum of the two
terms
\[ 
\Delta_L(i) \;:=\; \sum_{j\in I^*} g_j - \sum_{j\in Z^*\cap L(i)} g_j 
\quad\text{and}\quad
\Delta_R(i) \;:=\; -\sum_{j\in Z^*\cap R(i)} g_j.
\]
By the law of total probability and the linearity of expectation, we
have
\begin{equation}
\label{eq:decomposition-left-right}
\tag{$\ast\ast$}
E(\Delta_t) \;=\;\sum_{i\in I} 
 E(\Delta_L(i) \mid A_i) \cdot \Prob(A_i) 
+E(\Delta_R(i) \mid A_i)  \cdot \Prob(A_i). 
\end{equation}
In the following, the bits in~$R(i)$ are pessimistically assumed to
flip to~$1$ independently with probability~$p$ each if $A_i$ happens.
This leads to $E(\Delta_R(i) \mid A_i) \ge -p\sum_{j\in R(i)}g_j$.

In order to estimate $E(\Delta_L(i))$, we carefully inspect the
relation between the weights of the original function and the
potential function. By definition, we obtain $g_j/g_{k(i)} =
w_j/w_{k(i)}$ for $k(i)\le j\le i$ and $g_j/g_{k(i)} \le w_j/w_{k(i)}$
for $j>i$ whereas $g_{j}/g_{k(i)} \ge w_{j}/w_{k(i)}$ for $j<
k(i)$. Hence, if $A_i$ occurs then $g_j \ge g_{k(i)}\cdot
\frac{w_j}{w_{k(i)}}$ for $j\in I^*$ (since $i$ is the leftmost
flipping one-bit) whereas $g_j \le g_{k(i)}\cdot \frac{w_j}{w_{k(i)}}$
for $j\in L(i)$. Together, we obtain under~$A(i)$ the nonnegativity of
the random variable $\Delta_L(i)$:
\begin{align*}
\Delta_L(i) \mid A_i & \;=\; \sum_{j\in I^*\mid A_i} g_j - \sum_{j\in (Z^*\cap L(i))\mid A_i}  g_j \\
& \;\ge \;
\sum_{j\in I^*\mid A_i} g_{k(i)}\cdot \frac{w_j}{w_{k(i)}} - 
\sum_{j\in (Z^*\cap L(i))\mid A_i}  g_{k(i)}\cdot \frac{w_j}{w_{k(i)}} \;\ge\; 0
\end{align*}
using the definition of~$A_i$.

Now let $S_i:=\{\card{Z^*\cap L(i)}=0\}$ be the event that no zero-bit
from $L(i)$ flips. Using the law of total probability, we obtain that
\begin{multline*}
E(\Delta_L(i) \mid A_i) \cdot \Prob(A_i)  \;= \; 
E(\Delta_L(i) \mid A_i\cap S_i) \cdot \Prob(A_i\cap S_i) \\
 \;+\; E(\Delta_L(i) \mid A_i\cap \overline{S_i})  \cdot 
\Prob(A_i\cap \overline{S_i}).
\end{multline*}
Since $\Delta_L(i) | A_i \;\ge\; 0$, the conditional expectations are
non-negative. We bound the second term on the right-hand side
by~$0$. In conjunction with \eqref{eq:decomposition-left-right}, we
get
\[
E(\Delta_t) \;\ge\; 
\sum_{i\in I}  
 E(\Delta_L(i) \mid A_i\cap S_i) \cdot \Prob(A_i\cap S_i) 
+ E(\Delta_R(i) \mid A_i) \cdot  \Prob(A_i).
\]

Obviously, $E(\Delta_L(i) \mid A_i\cap S_i) \ge g_i$.  We estimate
$\Prob(A_i\cap S_i) \ge p (1-p)^{n-1}$ since it is sufficient to flip
only bit~$i$ and $\Prob(A_i)\le p$ since it is necessary to flip this
bit. Further above, we have bounded $E(\Delta_R(i) \mid A_i)$.  Taking
everything together, we get
\begin{align*}
E(\Delta_t) & 
\;\ge \; 
\sum_{i\in I} \left(p (1-p)^{n-1} g_{i} - p^2\sum_{j\in R(i)} g_j\right) \\
 & \;\ge\; \sum_{i\in I} \left(p (1-p)^{n-1} \frac{g_{i}}{g_{k(i)}}\gamma_{k(i)} - p^2\sum_{j=1}^{k(i)-1} \gamma_j\right).
\end{align*}
The term for $i$ equals
\vspace{-0.5\baselineskip}
\begin{align*}
& p (1-p)^{n-1}  \frac{g_{i}}{g_{k(i)}} \left(1+\frac{\alpha p}{(1-p)^{n-1}}\right)^{k(i)-1} 
- \frac{p^2\cdot \left(\left(1+\frac{\alpha p}{(1-p)^{n-1}}\right)^{k(i)-1} - 1\right)}{\left(\frac{\alpha p}{(1-p)^{n-1}}\right)} \\
& \;\ge\; \left(1-\frac{1}{\alpha}\right) p(1-p)^{n-1}  \frac{g_{i}}{g_{k(i)}}\left(1+\frac{\alpha p}{(1-p)^{n-1}}\right)^{k(i)-1} 
 \;= \; \left(1-\frac{1}{\alpha}\right) p (1-p)^{n-1} g_{i},
\end{align*}
where the inequality uses $g_{i}\ge g_{k(i)}$. Hence,
\[
E(\Delta_t) \; \ge \; 
\sum_{i\in I}
\left(1-\frac{1}{\alpha}\right) p (1-p)^{n-1} g_{i} 
\;=\;
\left(1-\frac{1}{\alpha}\right) p (1-p)^{n-1} g(a^{(t)}),
\]
which proves \eqref{eq:drift-general-case}, and, therefore, the
theorem.
\end{proofof}

\section{Refined Upper Bound for Mutation Probability $1/n$}
\label{sec:refined}
In this section, we consider the standard mutation probability $p=1/n$
and refine the result from Corollary~\ref{cor:constant-mut}.  More
precisely, we obtain that the lower order-terms are $O(n)$. The proof
will be shorter and uses a simpler potential function.

\begin{theorem}
\label{theo:refined}
\begin{sloppypar}
On any linear function, the expected optimization time of the \ea with
$p=1/n$ is at most $en\ln n + 2en + O(1)$, and the probability that
the optimization time exceeds $en\ln n + (1+t)en + O(1)$ is at most
$e^{-t}$.
\end{sloppypar}
\end{theorem}

\begin{proof}
Let $f(x)=w_nx_n+\dots+w_1x_1$ be the linear function at hand and let
$g(x)=g_nx_n+\dots+g_1x_1$ be the potential function defined by
\[
g_i = \left(1+\frac{1}{n-1}\right)^{\min\{j\le i \mid w_j=w_i\}-1},
\]
hence $g_i=(1+1/(n-1))^{i-1}$ for all~$i$ if and only if the $w_i$ are mutually distinct. 
We consider the stochastic process $X^{(t)}:=g(a^{(t)})$, where 
$a^{(t)}$ is the current search point of the \ea at time~$t$. Obviously, $X^{(t)}=0$ if and only if $f$ has been optimized.

Let $\Delta_t:= X^{(t)}-X^{(t+1)}$.  In a case analysis (partly
inspired by \citealp{DJWMultiplicativeDrift}), we will show below for
$n\ge 4$ that $E(\Delta_t\mid X^{(t)}= s)\ge s/(en)$. The initial
value satisfies
\begin{align*}
X^{(0)}& \;\le\; g_n+\dots+g_1 \;\le\;  \sum_{i=0}^{n-1} \left(1+\frac{1}{n-1}\right)^i 
\;=\; \frac{(1+1/(n-1))^n -1 }{1/(n-1)} \\
& \;\le\; (n-1) \left(1+\frac{1}{n-1}\right) e- (n-1) \le en,
\end{align*}
where we have used $(1+1/(n-1))^{n-1}\le e$. Hence, $\ln(X_0)\le (\ln
n) + 1$. Assuming $n\ge 4$, Theorem~\ref{theo:multdrift-uppper} yields
$E(T)\le en(\ln(n)+2)$ and $\Prob(T>en((\ln n)+t+1))\le e^{-t}$
regardless of the starting point, from which the theorem follows.

The case analysis fixes an arbitrary current search point
$a^{(t)}$. We denote by $I:=\{i\mid a^{(t)}_i=1\}$ the index set of
its one-bits and by $Z:=\{1,\dots,n\}\setminus I$ its zero-bits.  We
assume $I\neq \emptyset$ since there is nothing to show
otherwise. Denote by $a'$ the random (not necessarily accepted)
offspring produced by the \ea when mutating $a^{(t)}$ and by
$a^{(t+1)}$ the next search point after selection. Recall that
$a^{(t+1)}=a'$ if and only if $f(a')\le f(a^{(t)})$. In what follows,
we will often condition on the event $A$ that $a^{(t+1)}=a'\neq
a^{(t)}$ holds since $\Delta_t=0$ otherwise. Let $I^*:=\{i\in I\mid
a'_i=0\}$ be the set of flipped one-bits and by $Z^*:=\{i\in Z\mid
a'_i=1\}$ be the set of flipped zero-bits.  Note that
$I^*\neq\emptyset$ if $A$ occurs.

\textbf{Case 1:} Event $S_1:=\{\card{I^*}\ge 2\} \cap A$ occurs. Under
this condition, each zero-bit in~$a^{(t)}$ has been flipped to~$1$ in
$a^{(t+1)}$ with probability at most~$1/n$. Since $g_i\ge 1$ for $1\le
i\le n$, we have
\begin{align*}
E(\Delta_t \mid S_1) & \;\ge\; \card{I^*} - \frac{1}{n}\sum_{i\notin I} g_i 
\;\ge\; 2-\frac{1}{n}\sum_{i=1}^n \left(1+\frac{1}{n-1}\right)^{i-1}\\
& \;=\; 2-\frac{(1+1/(n-1))^{n}-1}{n/(n-1)} \;\ge\; 2-\left(e-\left(1-\frac{1}{n}\right)\right)\ge 0
\end{align*}
for $n\ge 4$, where we have used $1+1/(n-1)=1/(1-1/n)$.  Hence, we
pessimistically assume $E(\Delta_t \mid S_1)=0$.

\textbf{Case 2:} Event $S_2:=\{\card{I^*} = 1\} \cap A$ occurs.  Let
$i^*$ be single element of~$I^*$ and note that this is a random
variable.

\textbf{Subcase 2.1:} $S_{21}:=\{\card{I^*} = 1\} \cap \{Z^* =
\emptyset\} \cap A$ occurs.  Since $\{\card{I^*} = 1\}$ and $\{Z^* =
\emptyset\}$ together imply~$A$, the index~$i^*$ of the flipped
one-bit is uniform over~$I$.  Hence, $E(\Delta_t\mid S_{21}) =
\sum_{i\in I} g_i / \card{I}$.  Moreover, $\Prob(S_{21}) \ge
\card{I}(1/n)(1-1/n)^{n-1}\ge \card{I}/(en)$, implying $E(\Delta_t\mid
S_{21})\cdot \Prob(S_{21}) \ge g(a^{(t)})/(en) = X^{(t)}/(en)$. If we
can show that $E(\Delta_t\mid \{\card{I^*}=1\} \cap \{\card{Z^*}\ge
1\}\cap A)\ge 0$, which will be proven in Subcase~2.2 below, then
$E(\Delta_t\mid X^{(t)}=s)\ge s/(en)$ follows by the law of total
probability and the proof is complete.

\textbf{Subcase 2.2:} $S_{22}:=\{\card{I^*} = 1\} \cap \{\card{Z^*}\ge
1\} \cap A$ occurs. Let $j^*:=\max\{j\mid j\in Z^*\}$ be the index of
the leftmost flipping zero-bit, and note that also $j^*$ is
random. Since we work under $\card{I^*}=1$ and the $w_j$ are monotone
increasing \wrt~$j$, it is necessary for $A$ to occur that $w_{j^*}\le
w_{i^*}$ holds.

\textbf{Subcase 2.2.1:} $S_{221}:=\{\card{I^*} = 1\} \cap
\{\card{Z^*}\ge 1\} \cap\{j^*>i^*\} \cap A$ occurs. Then
$w_{j^*}=w_{i^*}$ and $\card{Z^*}=1$ must hold. In this case,
$g_{j^*}=g_{i^*}$ by the definition of~$g$ and $E(\Delta_t\mid
S_{221})=0$ follows immediately.

\textbf{Subcase 2.2.2:} $S_{222}:=\{\card{I^*} = 1\} \cap
\{\card{Z^*}\ge 1\} \cap\{j^*<i^*\} \cap A$ occurs.  If
$w_{j^*}=w_{i^*}$ then $\card{Z^*}=1$ must hold for $A$ to occur, and
zero drift follows as in the previous subcase. Now let us assume
$w_{j^*}<w_{i^*}$ and thus $g_{j^*}<g_{i^*}$.  For notational
convenience, we redefine $i^*:=\min\{i\mid w_i=w_{i^*}\}$.  We
consider $Z_r:=Z\cap \{1,\dots,i^*-1\}$, the set of potentially
flipping zero-bits right of~$i^*$, denote $k:=\card{Z_r}$ and note
that in the worst case, $Z_r=\{i^*-1,\dots,i^*-k\}$ as the $g_i$ are
non-decreasing.  By using $\tilde{p}:=\Prob(Z^*\cap Z_r\neq\emptyset)
= 1-(1-1/n)^{k}$ and the definition of conditional probabilities, we
obtain under $S_{222}$ that every bit from $Z_r$ is flipped (not
necessarily independently) with probability at most $(1/n)/\tilde{p} =
\frac{1}{n(1-(1-1/n)^k)}$.  We now assume that all the
corresponding~$a'$ are accepted. This is pessimistic for the following
reasons: Consider a rejected $a'$. If $\card{Z^*}=1$ then our
prerequisite $j^*<i^*$ and the monotonicity of the~$g_i$ imply a
negative $\Delta_t$-value.  If $\card{Z^*}>1$ then the negative
$\Delta_t$-value is due to the fact $g_i<g_{i-1}+g_{i-2}$ for $3\le
i\le n$.  Hence, using the linearity of expectation we get
\begin{align*}
& E(\Delta_t \mid S_{222} ) \;\ge\; g_{i^*} - \frac{1}{n\tilde{p}}\cdot \sum_{j\in Z_r} g_j 
\;\ge\; g_{i^*} - \sum_{j=1}^k 
\frac{g_{i^*-j}}{n(1-(1-1/n)^k)} \\
& \;=\; \left(1+\frac{1}{n-1}\right)^{i^*-1} - 
\sum_{j=0}^{k-1} 
\frac{(1+1/(n-1))^{i^*-1-j}}{n(1-(1-1/n)^k)} \\ 
& \; = \;
\left(1+\frac{1}{n-1}\right)^{i^*-k} 
\left( 
\left(1+\frac{1}{n-1}\right)^{k-1} - \frac{((1+1/(n-1))^k-1)\cdot (n-1)}{n(1-(1-1/n)^k)}
\right) 
\;=\; 0,
\end{align*}
where the last equality follows since $1+1/(n-1) = (1-1/n)^{-1}$ and 
\[
\frac{((1+1/(n-1))^k-1)\cdot (n-1)}{n(1-(1-1/n)^k)} = \left(1-\frac{1}{n}\right) 
\frac{(1-1/n)^{-k}-1}{1-(1-1/n)^k} = \left(1-\frac{1}{n}\right)^{1-k}.
\]
This completes the proof.
\end{proof}

\section{Lower Bounds}
\label{sec:lower}
In this section, we state lower bounds that prove the results from
Theorem~\ref{theo:general-upper} to be tight up to lower-order terms
for a wide range of mutation probabilities. Moreover, we show that the
lower bounds hold for the very large class of mutation-based
algorithms (Algorithm~\ref{alg:mutation-based}). Recall that a list of
the most important consequences is given above in
Theorem~\ref{thm:main-theorem}.  For technical reasons, we split the
proof of the lower bounds into two main cases, namely
$p=O(n^{-2/3-\epsilon})$ and $p=\Omega(n^{\epsilon-1})$ for any
constant~$\epsilon>0$.  Unless $p>1/2$, the proofs go back to \OneMax
as a worst case, as outlined in the following subsection.

\subsection{\OneMax as Easiest Linear Function}
\label{sec:onemax-easiest}
\citet{DJWLinearRevisited} show with respect to the \oneoneea with
standard mutation probability~$1/n$ that \OneMax is the ``easiest''
function from the class of functions with unique global optimum, which
comprises the class of linear functions. More precisely, the expected
optimization time on \OneMax is proved to be smallest within the
class.

We will generalize this result to $p\le 1/2$ with moderate additional
effort.  In fact, we will relate the behavior of an arbitrary
mutation-based EA on \ONEMAX to the \oneoneeamu in a similar way to
\citet[][Section~7]{SudholtLowerFitnessLevel}.  The latter algorithm,
displayed as Algorithm~\ref{alg:oneoneeamu}, creates search points
uniformly at random from time~$0$ to time~$\mu-1$ and then chooses a
best one from these to be the current search point at time~$\mu-1$;
afterwards it works as the standard \oneoneea.  Note that we obtain
the standard \oneoneea for $\mu=1$. Moreover, we will only consider
the case $\mu=\poly(n)$ in order to bound the running time of the
initialization. This makes sense since a unique optimum (such as the
all-zeros string for \OneMax) is with overwhelming probability not
found even when drawing $2^{\sqrt{n}}$ random search points.

\begin{algorithm}
\caption{\oneoneeamu}
\label{alg:oneoneeamu}
\begin{algorithmic}
\FOR{$t:=0 \to \mu-1$}
 \STATE choose  $x_t \in \{0,1\}^n$ uniformly at random.
 \ENDFOR
 \STATE $x_t:=\argmin\{f(x)\mid x\in\{x_0,\dots,x_t\}\}$ (breaking ties uniformly).
 \REPEAT 
  \STATE create $x'$ by flipping each bit in $x_t$ independently with prob.\ $p$.
  \STATE $x_{t+1}:=x'$ if $f(x') \le f(x_t)$, and $x_{t+1}:=x_t$ otherwise. 
  \STATE $t:=t+1$.
 \UNTIL{forever.}
\end{algorithmic}
\end{algorithm}

Our analyses need the monotonicity statement from
Lemma~\ref{lem:monotonicity} below, which is similar to Lemma~11 in
\citet{DJWLinearRevisited} and whose proof is already sketched in
\citet[][Section 5]{DJWSeal2000}. Note, however, that
\citet{DJWLinearRevisited} only consider $p=1/n$ and have a stronger
statement for this case. More precisely, they show
$\Prob(\ones{\textit{mut}(a)} = j) \ge \Prob(\ones{\textit{mut}(b)} =
j)$, which does not hold for large~$p$.  Here and hereinafter,
$\ones{x}$ denotes the number of ones in a bit string~$x$.

\begin{lemma}
\label{lem:monotonicity}
Let $a,b\in\{0,1\}^n$ be two search points satisfying
$\ones{a}<\ones{b}$.  Denote by $\textit{mut}(x)$ the random string
obtained by mutating each bit of~$x$ independently with
probability~$p$. Let $0\le j\le n$ be arbitrary. If $p\le 1/2$ then
\[
\Prob(\ones{\textit{mut}(a)} \le j) \;\ge\; \Prob(\ones{\textit{mut}(b)} \le j).
\]
\end{lemma}

\begin{proof}
We prove the result only for $\ones{b}=\ones{a}+1$. The general
statement then follows by induction on $\ones{b}-\ones{a}$.

By the symmetry of the mutation operator,
$\Prob(\ones{\textit{mut}(x)}\le j)$ is the same for all $x$ with
$\ones{x}=\ones{a}$. We therefore assume $b\ge a$ (\ie, $b$ is
component-wise not less than~$a$). In the following, let $s^*$ be the
unique index where $b_{s^*}=1$ and $a_{s^*}=0$. Let $S(x)$ be the
event that bit~$s^*$ flips when $x$ is mutated. Since bits are flipped
independently, it holds $\Prob(S(x))=p$ for any~$x$. We write
$a':=\textit{mut}(a)$ and $b':=\textit{mut}(b)$. Assuming $p\le 1/2$,
the aim is to show $\Prob(\ones{a'}\le j) \ge \Prob(\ones{b'}\le j)$,
which by the law of total probability is equivalent to
\begin{align}
\notag
& \biggl(\Prob(\ones{a'}\le j\mid \overline{S(a)}) -\Prob(\ones{b'}\le j\mid \overline{S(b)})\biggr)\cdot (1-p) 
\\
\label{eq:monotonicity}
\tag{$\ast$}
& \quad + \biggl(\Prob(\ones{a'}\le j\mid S(a))-\Prob(\ones{b'}\le j\mid S(b))\biggr)\cdot p 
 \;\ge\; 0. 
\end{align}
Note that the relation $\Prob(\ones{a'}\le j\mid \overline{S(a)}) \ge
\Prob(\ones{b'}\le j\mid \overline{S(b)})$ follows from a simple
coupling argument as $a'\le b'$ holds if the mutation operator flips
the bits other than~$s^*$ in the same way with respect to $a$
and~$b$. Moreover,
\begin{multline*}
\Prob(\ones{a'}\le j\mid \overline{S(a)}) - \Prob(\ones{b'}\le j\mid \overline{S(b)}) 
\\ \;=\;
\Prob(\ones{b'}\le j\mid S(b)) - \Prob(\ones{a'}\le j\mid S(a)) 
\end{multline*}
since $a$ is obtained from~$b$ by flipping bit~$s^*$ and vice
versa. Hence, \eqref{eq:monotonicity} follows.
\end{proof}

The following theorem is a generalization of Theorem~9 by
\citet{DJWLinearRevisited} to the case~$p\le 1/2$ instead of
$p=1/n$. However, we not only generalize to higher mutation
probabilities, but also also consider the more general class of
mutation-based algorithms. Finally, we prove stochastic ordering,
while \citet{DJWLinearRevisited} inspect only the expected
optimization times. Still, many ideas of the original proof can be
taken over and be combined with the proof of Theorem~5 in
\citet{SudholtLowerFitnessLevel}.

\begin{theorem}
\label{theo:onemax-easiest}
Consider a mutation-based EA $A$ with population size~$\mu$ and
mutation probability $p\le 1/2$ on any function with a unique global
optimum. Then the optimization time of~$A$ is stochastically at least
as large as the optimization time of the \oneoneeamu on \OneMax.
\end{theorem}

\newcommand{\tonemax}{T_{\OneMax}}
\newcommand{\calx}{\mathcal{X}}
\begin{proof}
Let $f$ denote the function with unique global optimum, which we
\wlo\ assume to be the all-zeros string. For any sequence
$\calx=(x_0,\dots,x_{\ell-1})$ of search points over~$\{0,1\}^n$, let
$q(\calx)$ be the probability that~$\calx$ represents the first~$\ell$
search points $x_0,\dots,x_{\ell-1}$ created by Algorithm~$A$ on~$f$
(its so-called history up to time~$\ell-1$).  For any history $\calx$
with $q(\calx)>0$, let $T_f(\calx)$ denote the random optimization
time of Algorithm~$A$ on~$f$, given that its history up to time~$\ell$
equals~$\calx$.  Let 
\[
\Xi_\ell:=\left\{\calx=(x_0,\dots,x_{\ell-1})\in \bigtimes_{i=1}^\ell
\{0,1\}^{n} \Bigm| q(\calx)>0\right\}
\] 
denote the set of all possible histories of length~$\ell$ with respect
to Algorithm~$A$ on~$f$, and let $\Xi:=\{\bigcup_{\ell=1}^m \Xi_\ell
\mid m\in \N\}$ denote all possible histories of finite
length. Finally, for any $\calx\in\Xi$, let $L(\calx)$ denote the
length of~$\calx$.

Given any $\calx\in \Xi$, let \oneoneeaarg{\calx} be the algorithm
that chooses a search point with minimal number of ones from $\calx$
as current search point at time~$L(\calx)-1$ (breaking ties uniformly)
and afterwards proceeds as the standard \oneoneea on \OneMax.  Now,
let $\tonemax(\calx)$ denote the random optimization time of the
\oneoneeaarg{\calx}.  We claim that the stochastic ordering
\begin{equation*}
\Prob(T_f(\calx)\ge t) \;\ge\; \Prob(\tonemax(\calx)\ge t)
\end{equation*}
holds for every $\calx\in\Xi$ satisfying $L(\calx)\ge \mu$ and every
$t\ge 0$. Note that the random vector of initial search points
$\calx^*:=(x_0,\dots,x_{\mu-1})$ follows the same distribution in both
Algorithm~$A$ and the \oneoneeamu.  In particular, the two algorithms
are identical before time~$\mu-1$, \ie, before initialization is
finished. Furthermore, \oneoneeaarg{\calx^*} is the \oneoneeamu
initialized with $\calx^*$.  Altogether, the claimed stochastic
ordering implies the theorem. Moreover, regardless of the length
$L(\calx)$, the claim is obvious for $t\le L(\calx)$ since the
behavior up to time $L(\calx)$ is fixed.

For any $\calx\in \Xi$, let $\ones{\calx} := \min\{\ones{x}\mid x\in
\calx\}$ denote the best number of ones in the history, where $x\in
(x_0,\dots,x_{\ell-1})$ means that $x=x_i$ for some
$i\in\{0,\dots,\ell-1\}$.  For every $k\in\{0,\dots,n\}$, every
$\ell\ge \mu$ and every $t\ge 0$, let
\[
p_{k,\ell}(t) \;:= \; \min\{\Prob(\tonemax(\calx) \ge \ell+t)\mid \calx\in\Xi_\ell, \ones{\calx} = k\}
\]
be the minimum probability of the \oneoneeaarg{\calx} needing at least
$\ell+t$ steps to optimize \OneMax from a history of length~$\ell$
whose best search point has exactly $k$ one\nobreakdash-bits.  Due to
the symmetry of the \OneMax function and the definition of
\oneoneeaarg{\calx}, we have ${\Prob(\tonemax(\calx)\ge \ell+t)} =
p_{k,\ell}(t)$ for every~$\calx$ satisfying $\card{\calx}=\ell$ and
$\ones{\calx} = k$. In other words, the minimum can be omitted from
the definition of $p_{k,\ell}$.

Furthermore, for every $k\in\{0,\dots,n\}$, every $\ell\ge \mu$ and
every $t\ge 0$, let
\[
\tp_{k,\ell}(t) \;:= \; \min\{\Prob(T_f(\calx) \ge \ell+t)\mid \calx\in\Xi_\ell, \ones{\calx} \ge k\}
\]
be the minimum probability of Algorithm~$A$ needing at least $\ell+t$
steps to optimize~$f$ from a history of length~$\ell\ge\mu$ whose best
search point has \emph{at least} $k$ one-bits. We will show
$\tp_{k,\ell}(t)\ge p_{k,\ell}(t)$ for any $k\in\{0,\dots,n\}$ and
$\ell\ge \mu$ by induction on~$t$. In particular, by choosing
$\ell:=\mu$ and applying the law of total probability with respect to
the outcomes of~$\ones{\calx^*}$, this will imply the above-mentioned
stochastic ordering and, therefore, the theorem.
 
If $k\ge 1$ then $p_{k,\ell}(0)=\tp_{k,\ell}(0)=1$ for any~$\ell\ge
\mu$ since the condition means that the first~$\ell$ search points do
not contain the optimum.  Moreover, $p_{0,\ell}(t)=\tp_{0,\ell}(t)=0$
for any $t\ge 0$ and~$\ell\ge \mu$ since a history beginning with the
all-zeros string corresponds to optimization time~$0$ and thus
minimizes both $\Prob(T_f(\calx) \ge t+\ell)$ and
$\Prob(\tonemax(\calx) \ge t+\ell)$.  Now let us assume that there is
some $t\ge 0$ such that $\tp_{k,\ell}(t') \ge p_{k,\ell}(t')$ holds
for all $0\le t'\le t$, $k\in\{0,\dots,n\}$, and $\ell\ge \mu$. Note
that the inequality has already been proven for all~$t$ if $k=0$.
 
Consider the \oneoneeaarg{\calx} for an arbitrary $\calx$ satisfying
$L(\calx)=\ell\ge \mu$ and $\ones{\calx}=k+1$ for some
$k\in\{0,\dots,n-1\}$.  Let some $x\in\{0,1\}^n$, where
$\ones{x}=k+1$, be chosen from~$\calx$ and let $y\in\{0,1\}^n$ be the
random search point generated by flipping each bit in~$x$
independently with probability~$p$. The \oneoneeaarg{\calx} will
accept~$y$ as new search point at time~$\ell+1 > \mu$ if and only if
$\ones{y}\le \ones{x}=k+1$. Hence,
\begin{equation}
\label{eq:dominance-onemax}
\tag{$\ast$}
p_{k+1,\ell}(t+1) \;=\; \Prob(\ones{y}\ge k+1)\cdot p_{k+1,\ell+1}(t) + \sum_{j=0}^k \Prob(\ones{y} = j)\cdot p_{j,\ell+1}(t).
\end{equation}

Next, let $\calx$, where again $L(\calx)=\ell\ge \mu$, be a history
satisfying $\Prob(T_f(\calx)\ge t+1) = \tp_{k+1,\ell}(t+1)$ and let
$\tilde{x}$ be the (random) search point that is chosen for mutation
at time~$\ell$ in order to obtain the equality of the two
probabilities. Note that $\ones{\tx}\ge k+1$. Moreover, let
$\ty\in\{0,1\}^n$ be the random search point generated by flipping
each bit in~$\tx$ independently with probability~$p$. Let $\calx'$ be
the concatenation of~$\calx$ and~$\ty$. Then
 \begin{multline*}
 \tp_{k+1,\ell}(t+1) \;=\; 
 \Prob(\ones{\ty}\ge k+1)\cdot \Prob(T_f(\calx') \ge t\mid \ones{\ty}\ge k+1) \\ + \sum_{j=0}^k \Prob(\ones{\ty} =j) 
 \cdot \Prob(T_f(\calx') \ge t\mid \ones{\ty} = j),
 \end{multline*}
which, by definition of the $\tp_{i}(t)$, gives us the lower bound
 \[
 \tp_{k+1,\ell}(t+1) \;\ge\; \Prob(\ones{\ty}\ge k+1)\cdot \tp_{k+1,\ell+1}(t)   + \sum_{j=0}^k \Prob(\ones{\ty}=j) 
 \cdot \tp_{j,\ell+1}(t).
 \]
 
To relate the last inequality to \eqref{eq:dominance-onemax} above, we
interpret the right-hand side as a function of $k+2$ variables. More
precisely, let $\phi(a_0,\dots,a_{k+1}):=\sum_{j=0}^{k+1} a_j
\tp_{j,\ell+1}(t)$ and consider the vectors
\[
 v^{(f)}=(v^{(f)}_0,\dots,v^{(f)}_{k+1}) := (\Prob(\ones{\ty}=0),
 \dots, \Prob(\ones{\ty}=k), \Prob(\ones{\ty}\ge k+1))
\]
and 
\[
 v^{(O)} = (v^{(O)}_0,\dots,v^{(O)}_{k+1}) :=(\Prob(\ones{y}=0),
 \dots, \Prob(\ones{y}=k), \Prob(\ones{y}\ge k+1)).
\]
If we can show that $\phi(v^{(f)})\ge \phi(v^{(O)})$, then we can
conclude
\begin{multline*}
\tp_{k+1,\ell}(t+1) \;\ge\; \phi(v^{(f)}) \;\ge\; \phi(v^{(O)}) \\
\;\ge\;
\Prob(\ones{y}\ge k+1)\cdot p_{k+1,\ell+1}(t) + \sum_{j=0}^k \Prob(\ones{y} = j)\cdot p_{j,\ell+1}(t) 
\;=\; p_{k+1,\ell}(t+1),
\end{multline*}
where the last inequality follows from the induction hypothesis and
the equality is from \eqref{eq:dominance-onemax}. This will complete
the induction step.
 
To show the outstanding inequality, we use that for $0\le j\le k$
\[
	\Prob(\ones{y}\le j) 
	\;\ge\; 
 	\Prob(\ones{\ty}\le j),
\]
which follows from Lemma~\ref{lem:monotonicity} since $\ones{\tx}\ge
\ones{x}$ and $p\le 1/2$. In other words,
\[
\sum_{i=0}^j v^{(O)}_i \;\ge\; \sum_{i=0}^j v^{(f)}_i
\]
for $0\le j\le k$ and $\sum_{i=0}^{k+1} v^{(O)}_i = \sum_{i=0}^{k+1}
v^{(f)}_i$ since we are dealing with probability
distributions. Altogether, the vector $v^{(O)}$ majorizes the vector
$v^{(f)}$.  Since they are based on increasingly restrictive
conditions, the $\tp_j(t)$ are non-decreasing in~$j$. Hence, $\phi$ is
Schur-concave (cf.\ Theorem~A.3 in Chapter~3 of
\citealp{MarshallInequalities}), which proves $\phi(v^{(f)})\ge
\phi(v^{(O)})$ as desired.
\end{proof}

\subsection{Large Mutation Probabilities}
It is not too difficult to show that mutation probabilities
$p=\Omega(n^{\epsilon-1})$, where $\epsilon>0$ is an arbitrary
constant, make the \oneoneea (and also the \oneoneeamu) flip too many
bits for it to optimize linear functions efficiently.

\begin{theorem}
\label{theo:lower-high-p}
On any linear function, the optimization time of an arbitrary
mutation-based EA with $\mu=\poly(n)$ and $p=\Omega(n^{\epsilon-1})$
for some constant $\epsilon>0$, is bounded from below by
$2^{\Omega(n^{\epsilon})}$ with probability
$1-2^{-\Omega(n^{\epsilon})}$.
\end{theorem}

\begin{proof}
Due to Theorem~\ref{theo:onemax-easiest}, if suffices to show the
result for the \oneoneeamu on \OneMax. The following two statements
follow from Chernoff bounds (and a union bound over $\mu=\poly(n)$
search points in the second statement).
\begin{enumerate}
\item 
Due to the lower bound on~$p$, the probability of a single step not
flipping at least $\lfloor pi/2\rfloor$ bits out of a set of~$i$ bits
is at most $2^{-\Omega(pi)}=2^{-\Omega(in^{\epsilon-1})}$.
\item The search point~$x_{\mu-1}$ has at least $n/3$ and at most
  $2n/3$ one-bits with probability $1-2^{-\Omega(n)}$.
\end{enumerate}
Furthermore, as we consider \OneMax, the number of one-bits is
non-increasing over time. We assume an $x_{\mu-1}$ being non-optimal
and having at most $2n/3$ one-bits, which contributes a term of only
$2^{-\Omega(n)}$ to the failure probability. The assumption means that
all future search points accepted by the \oneoneeamu will have at
least $n/3$ zero-bits. In order to reach the optimum, none of these is
allowed to flip. As argued above, the probability of this happening is
$2^{-\Omega(n^{\epsilon})}$, and by the union bound, the total
probability is still $2^{-\Omega(n^{\epsilon})}$ in a number of
$2^{cn^{\epsilon}}$ steps if the constant $c$ is chosen small enough.
\end{proof}

Mutation-based EAs have only been defined for $p\le 1/2$ since
flipping bits with higher probability seems to contradict the idea of
a mutation. However, for the sake of completeness, we also analyze the
\oneoneea with $p>1/2$ and obtain exponential expected optimization
times. Note that we do not know whether \OneMax is the easiest linear
function in this case.

\begin{theorem}
\label{theo:lower-p-larger-half}
\begin{sloppypar}
On any linear function, the expected optimization time of the
\oneoneea with mutation probability $p>1/2$ is bounded from below by
$2^{\Omega(n)}$.
\end{sloppypar}
\end{theorem}

\begin{proof}
We distinguish between two cases.

\textit{Case 1: $p\ge 3/4$.} Here we assume that the initial search
point has at least $n/2$ leading zeros and is not optimal, the
probability of which is at least $2^{-n/2-1}$. Since the $n/2$ most
significant bits are set correctly in this search point, all accepted
search points must have at least $n/2$ zeros as well. To create the
optimum, it is necessary that none of these flips. This occurs only
with probability at most~$(1/4)^{n/2}$, hence the expected
optimization time under the assumed initialization is at least
$4^{n/2}$. Altogether, the unconditional expected optimization time is
at least $2^{-n/2-1}\cdot 4^{n/2}=2^{\Omega(n)}$.

\textit{Case 2: $1/2<p\le 3/4$.} Now the aim is to show that all
created search points have a number of ones that is in the interval
$I:=[n/8,7n/8]$ with probability $1-2^{-\Omega(n)}$. This will imply
the theorem by the usual waiting time argument.

Let $x$ be a search point such that $\ones{x}\in I$. We consider the
event of mutating $x$ to some~$x'$ where $\ones{x'}<n/8$. Since
$p>1/2$, this is most likely if $\ones{x}=7n/8$ (using the ideas
behind Lemma~\ref{lem:monotonicity} for the complement of~$x$). Still,
using Chernoff bounds and $p\le 3/4$, at least $(1/5)\cdot (7n/8) >
n/8$ one-bits are not flipped with probability $1-2^{-\Omega(n)}$.  By
a symmetrical argument, the probability is $2^{-\Omega(n)}$ that
$\ones{x'}>7n/8$.
\end{proof}

As was to be expected, no polynomial expected optimization times were
possible for the range of~$p$ considered in this subsection.

\subsection{Small Mutation Probabilities}
We now turn to mutation probabilities that are bounded from above by
roughly $1/n^{2/3}$. Here relatively precise lower bounds can be
obtained.

\begin{theorem}
\label{theo:general-lower}
\begin{sloppypar}
On any linear function, the expected optimization time of an arbitrary
mutation-based EA with $\mu=\poly(n)$ and $p=O(n^{-2/3-\epsilon})$ is
bounded from below by \[(1-o(1))(1-p)^{-n} (1/p) \min\{\ln
n,\ln(1/(p^3n^2))\}.\]
\end{sloppypar}
\end{theorem}

As a consequence from Theorem~\ref{theo:general-lower}, we obtain that
the bound from Theorem~\ref{theo:general-upper} is tight (up to
lower-order terms) for the \oneoneea as long as $\ln(1/(p^3n^2))=\ln n
- o(\ln n)$. This condition is weaker than $p=O((\ln n)/n)$. If
$p=\omega((\ln n)/n)$ or $p=o(1/\poly(n))$, then
Theorem~\ref{theo:general-lower} in conjunction with
Theorem~\ref{theo:lower-high-p} and \ref{theo:lower-p-larger-half}
imply superpolynomial expected optimization time. Thus, the bounds are
tight for all~$p$ that allow polynomial optimization times.

Before the proof, we state another important consequence, implying the
statement from Theorem~\ref{thm:main-theorem} that using the \oneoneea
with mutation probability $1/n$ is optimal for any linear function.

\begin{corollary}
\label{cor:constant-mut-lower}
\begin{sloppypar}
On any linear function, the expected optimization time of a
mutation-based EA with $\mu=\poly(n)$ and $p=c/n$, where $c>0$ is a
constant, is bounded from below by $(1-o(1)) ((e^c/c) n\ln n)$.  If
$p=\omega(1/n)$ or $p=o(1/n)$, the expected optimization time is
$\omega(n\ln n)$.
\end{sloppypar}
\end{corollary}

\begin{proof}
The first statement follows immediately from
Theorem~\ref{theo:general-lower} using $(1-c/n)^{-n} \ge e^c$ and
$\ln(1/(p^3 n^2))=\ln n -O(\ln c)$.  The second one follows, depending
on~$p$, either from Theorem~\ref{theo:lower-high-p} or, in that case
assuming $p=O((\ln n)/n)$, from Theorem~\ref{theo:general-lower},
noting that $(1-p)^{-n}(1/p) \ge e^{np}/p = \omega(n)$ if
$p=\omega(1/n)$ or $p=o(1/n)$.
\end{proof}

Recall that by Theorem~\ref{theo:onemax-easiest}, it is enough to
prove Theorem~\ref{theo:general-lower} for the \oneoneeamu on
\OneMax. As mentioned above, this is a well-studied function, for
which strong upper and lower bounds are known in the case $p=1/n$. Our
result for general~$p$ is inspired by the proof of Theorem~1 in
\citet{DFW10}, which uses an implicit multiplicative drift theorem for
lower bounds. Therefore, we now need an upper bound on the
multiplicative drift, which is given by the following generalization
of Lemma~6 in \cite{DFW11}.

\begin{lemma}
\label{lem:exp-progress-upper}
Consider the \oneoneea with mutation probability~$p$ for the
minimization of $\OneMax$. Given a current search point with $i$
one-bits, let $I'$ denote the random number of one-bits in the
subsequent search point (after selection). Then we have $E[i-I']\le
ip(1-p+ip^2/(1-p))^{n-i}$.
\end{lemma}

\begin{proof}
Note that $I'\le i$ since the number of one-bits in the process is
non-increasing.  Hence, only mutations that flip at least as many
one-bits as zero-bits have to be considered. The event that the total
number of one-bits is decreased by~$k\ge 0$ can be partitioned into
the subevents~$F_{k,j}$ that $k+j$ one-bits and $j$ zero-bits flip,
for all $j\in \mathds{Z}^{+}_0$.  The probability of an individual
event~$F_{k,j}$ equals
\[
\binom{i}{k+j}\binom{n-i}{j} p^{k+2j}\left(1-p\right)^{n-k-2j},
\]
where $\binom{a}{b} := 0$ for $b>a$. Thus, we have
\begin{align*}
E(i-I')
  &\;\le\; 
\sum_{k=1}^{i} k \sum_{j\ge 0} \binom{i}{k+j}\binom{n-i}{j} p^{k+2j}(1-p)^{n-k-2j}\\
& \;\le\; \underbrace{\sum_{k=1}^{i} k \binom{i}{k} p^{k}(1-p)^{n-k}}_{=:S_1}
  \cdot \underbrace{\sum_{j=0}^{n-i} i^j \binom{n-i}{j} \left(\frac{p}{1-p}\right)^{2j}}_{=:S_2},
\end{align*}
where the second inequality uses $\binom{i}{k+j} \le i^j\cdot
\binom{i}{k}$. Factoring out $(1-p)^{n-i}$ of $S_1$, we recognize the
expected value of a binomial distribution with parameters $i$ and $p$,
which means $S_1=(1-p)^{n-i}\cdot ip$. Regarding~$S_2$, we apply the
Binomial Theorem and obtain $S_2=(1+i(p/(1-p))^2)^{n-i}$. The product
of $S_1$ and $S_2$ is the upper bound from the lemma.
\end{proof}

\begin{proofof}{Theorem~\ref{theo:general-lower}}
As already mentioned, we may assume that the linear function is
\OneMax and that the algorithm is the \oneoneeamu.  The idea is to
apply Theorem~\ref{theo:multdrift-lower}, which is the above-mentioned
multiplicative drift theorem for lower bounds, for a suitable choice
of the parameters. Let $\tp:=\max\{p,1/n\}$. We first observe that the
probability of flipping at least $b:=\tp n\ln n$ bits in a single step
is bounded from above by
\[
\binom{n}{\tp n\ln n}\cdot p^{\tp n\ln n} \;\le\; 
\left(\frac{e \tp n}{\tp n\ln n}\right)^{\tp n\ln n} 
\;=\; 2^{-\Omega(\tp n(\ln n)(\ln\ln n))},
\]
where we have used $p\le \tp$. Hence, the probability is
superpolynomially small.  In the following, we assume that the number
of one-bits changes by at most $b$ in each of a total number of at
most $(1-p)^{-n}n\ln n = 2^{O(\tp n)+O(\ln\ln n)}$ steps that are
considered for the lower bound we want to prove. This event holds with
probability $1-o(1)$, which, using the law of total probability,
decreases the bound only by a factor of $1-o(1)$.

Let $X^{(t)}$ denote the number of one-bits at time~$t$ and note 
that this is non-increasing over time.  We choose
$\smin:=n\tp\ln^2 n$ and $\beta:=1/\!\ln n$ and introduce
$\cmax:=1/(2\tp^2 n\ln n)$ as an additional upper bound.  Note that
$\cmax\le n/(2\ln n)$ due to $\tp\ge 1/n$. Since the $\mu$ initial
search points are drawn uniformly at random and $\mu=\poly(n)$, it
holds $X_{\mu}\ge \cmax$ with probability $1-o(1)$. Again, assuming
this to happen, we lose a factor $1-o(1)$ in the bound we want to
prove. Moreover, due to our assumption $p=O(n^{-2/3-\epsilon})$ (which
means $\tp=O(n^{-2/3-\epsilon})$), we have $b=n\tp\ln n\le 1/(4\tp^2
n\ln n) = \cmax/2$ for $n$ large enough.  Altogether, it holds
$\cmax/2 \le X_{t^*}\le \cmax$ at the first point of time~$t^*$ where
$X_{t^*}\le \cmax$. To simplify issues, we consider the process only
from time~$t^*$ on. Skipping the first~$t^*$ steps, we pessimistically
assume $s_0:=\cmax/2$ as starting point and $X^{(t)}\le\cmax$ for all
$t\ge 0$.  The second condition of the drift theorem is now fulfilled
since the bound on~$\tp$ also implies $b = \tp n\ln n \le 1/(2\tp^2n
\ln^2 n) = \beta \cmax$, where $\beta \cmax$ is the largest value for
$\beta s$ to be taken into account.

Assembling the factors from the lower bound in
Theorem~\ref{theo:multdrift-lower}, we get $\frac{1-\beta}{1+\beta} =
1-o(1)$. Furthermore, we have $\ln(s_0/\smin) = \ln(1/(4\tp^3n^2 \ln^3
n)) = \ln(1/(\tp^3 n^2)) - O(\ln\ln n)$, which is $(1-o(1))
\ln(1/(\tp^3 n^2))$ by our assumption on $\tp$. If we can prove that
$1/\delta = (1-o(1))(1-p)^{-n} (1/p)$, the proof is complete.

To bound $\delta$, we use Lemma~\ref{lem:exp-progress-upper}. Note
that $i\le \cmax$ holds in our simplified process. Using the lemma and
recalling that $1/\tp \le 1/p$, we get
\begin{multline*}
  \frac{E(X^{(t)}-X^{(t+1)}\mid X^{(t)} = i)}{i} \;\le\; p\left(1-p+\frac{\cmax p^2}{1-p}\right)^{n-\cmax}\\ 
\;\le\; p\left(1-p+\frac{1}{n\ln n}\right)^{n-\cmax} 
\;\le\; p\left((1-p)\left(1+\frac{2}{n\ln n}\right)\right)^{n-\cmax} \\
\;=\; (1+o(1)) p (1-p)^n,
\end{multline*}
where we have used $p\le 1/2$ and $(1+2/(n\ln n))^{n}=1+o(1)$ and
$(1-p)^{-\cmax} = (1-p)^{-1/(2\tp^2 n\ln n)}=1+o(1)$.  Hence,
$1/\delta\ge (1-o(1))(1/p)(1-p)^{-n}$ as suggested, which completes
the proof.
\end{proofof}

Finally, we remark that the expected optimization time of the
\oneoneea with $p=1/n$ on \OneMax is known to be $en\ln n-\Theta(n)$
\citep{DFW11}. Hence, in conjunction with Theorems~\ref{theo:refined}
and~\ref{theo:onemax-easiest}, we obtain for $p=1/n$ that the expected
optimization time of the \oneoneea varies by at most an additive term
$\Theta(n)$ within the class of linear functions.

\section*{Conclusions}
We have presented new bounds on the expected optimization time of the
\ea on the class of linear functions. The results are now tight up to
lower-order terms, which applies to any mutation probability $p=O((\ln
n)/n)$. This means that $1/n$ is the optimal mutation probability on any
linear function.  We have for the first time studied the case
$p=\omega(1/n)$ and proved a phase transition from polynomial to
exponential running time in the regime $\Theta((\ln n)/n)$. The lower
bounds show that \OneMax is the easiest linear function for all $p\le
1/2$, and they apply not only to the \oneoneea but also to the large
class of mutation-based EAs. They so exhibit the \oneoneea as optimal
mutation-based algorithm on linear functions. The upper bounds hold
with high probability. As proof techniques, we have employed
multiplicative drift in conjunction with adaptive potential
functions. In the future, we hope to see these techniques applied to
the analysis of other randomized search heuristics.

We finish with an open problem. Even though our proofs of upper bounds
would simplify for the function \BV, this function is often considered
as a worst case. Is it true that the runtime of the \ea on \BV is
stochastically largest within the class of linear functions, thereby
complementing the result that the runtime on \OneMax is stochastically
smallest?

\section*{Acknowledgments}
The author thanks Benjamin Doerr, Timo K\"otzing, Per Kristian Lehre
and Carola Winzen for constructive discussions on the subject and for
ruling out an early proof attempt. Moreover, he thanks Daniel
Johannsen for proofreading a draft of the Sections~\ref{sec:upper}
and~\ref{sec:refined} and pointing out a simplification of the proof
of Theorem~\ref{theo:general-upper}. Finally, he thanks Dirk Sudholt, 
who suggested to study the class of mutation-based EAs.

\begin{appendix}
\section{Multiplicative Drift for Lower Bounds}
In this appendix, we supply the proof of Theorem
\ref{theo:multdrift-lower}, the lower-bound version of the
multiplicative drift theorem.  The proof follows the one of Theorem~5
in \citet{LWECCC10} and uses the following additive drift theorem.
\begin{theorem}[\cite{Jens2007ES}]\label{thm:pol-drift}
  Let $X^{(1)}, X^{(2)},\dots$ be random variables with bounded support
  and let $T$ be the stopping time defined by 
  $T:=\min\{t\mid X^{(1)}+\dots+X^{(t)}\ge g\}$ for a given $g>0.$ If
  $\expect{T}$ exists and $\expect{X^{(i)}\mid T\ge i}\le u$ for
  $i\in\mathds{N}$, then $\expect{T}\ge g/u$.
\end{theorem}

The proof of Theorem~\ref{theo:multdrift-lower} also makes use of the
following simple lemma.
\begin{lemma}\label{lem:cond-exp-bound}
  Let $X$ be any random variable, and $k$ any real number. If 
  it holds that $\prob{X<k}>0$, then $\expect{X}\ge \expect{X\mid X<k}$.
\end{lemma}
\begin{proof}
  Define $p:=\prob{X<k}$ and $\mu_k:=\expect{X\mid X< k}$. The lemma
  clearly holds when $p=1$ such that we assume $0<p<1$ in the
  following. If $\expect{X}$ is positive infinite then $\expect{X}\ge
  \mu_k$ is obvious. If $\expect{X}$ is negative infinite then so is
  $\mu_k$ by the law of total probability. Finally, for finite
  $\expect{X}$, the law of total probability yields
\begin{align*}
  \expect{X}
  & \;=\; (1-p)\cdot \expect{X\mid X\ge k} + p\cdot \mu_k
   \;\ge\; (1-p)\cdot k + p\cdot \mu_k\\
  & \;>\; (1-p)\cdot \mu_k + p\cdot \mu_k
    \;=\; \expect{X\mid X< k}.
\end{align*}
\end{proof}

\begin{proofof}{Theorem~\ref{theo:multdrift-lower}}
The proof generalizes the proof of Theorem~1 in \citet{DFW10}.  The
random variable $T$ is non-negative. Hence, if the expectation of $T$
does not exist, then it is positive infinite and the theorem holds.
We condition on the event $T>t$, but we omit stating this event in the
expectations for notational convenience.  We define the stochastic
process $Y^{(t)}:=\ln(X^{(t)})$ (note that $X^{(t)}\ge 1$), and apply
Theorem~\ref{thm:pol-drift} with respect to the random variables
\begin{align*}
  \Delta_{t+1}(s) \;:=\; \left(Y^{(t)}-Y^{(t+1)}\mid X^{(t)}=s\right) \;=\; 
  \left(\mathord{\ln}\mathord{\left(\frac{s}{X^{(t+1)}}\right)}\mid  X^{(t)}=s\right).
\end{align*}
We consider the time until $X^{(t)}\le \smin$ if $X^{(0)} = s_0$ and
use the parameter $g:=\ln (s_0/\smin)$. By the law of total
probability, the expectation of $\Delta_{t+1}(s)$ can be expressed
as
\begin{multline}
    \prob{s-X^{(t+1)}\ge \beta s}\cdot
      \expect{\Delta_{t+1}(s)\mid s-X^{(t+1)}\ge \beta s} \\
      + \prob{s-X^{(t+1)} < \beta s}\cdot
        \expect{\Delta_{t+1}(s)\mid s-X^{(t+1)} < \beta s}.
\label{eq:2}
\end{multline}
By applying the second condition from the theorem, the first term
in~\eqref{eq:2} can be bounded from above by $\frac{\beta\delta}{\ln
  s} \cdot \ln s = \beta\delta$. The logarithmic function is
concave. Hence, by Jensen's inequality, the second term
in~\eqref{eq:2} is at most
\begin{multline*}
  \mathord{\ln} \mathord{\left(\expect{\frac{s}{X^{(t+1)}}\mid s-X^{(t+1)} < \beta s \wedge X^{(t)}=s}\right)}\\
   \;=\; \mathord{\ln}\mathord{\left(1+\expect{\frac{s-X^{(t+1)}}{X^{(t+1)}}\mid s-X^{(t+1)} < \beta s\wedge
     X^{(t)}=s}\right)}.
\end{multline*}

By using the inequality $\ln(1+x)\le x$ as well as the conditions $X_{t+1}\ge
(1-\beta) s$ and $X_{t+1}\le X_t$, this simplifies to
\begin{multline*}
  \expect{\frac{s-X^{(t+1)}}{X^{(t+1)}}\mid s-X^{(t+1)} < \beta s \wedge X^{(t)}=s}\\
\;<\;   \expect{\frac{s-X^{(t+1)}}{(1-\beta) s}\mid s-X^{(t+1)} < \beta s \wedge X^{(t)}=s}.
\end{multline*}
By Lemma~\ref{lem:cond-exp-bound} and the first condition from the
theorem, it follows that the second term in~\eqref{eq:2} is at most
\begin{align*}
  \expect{\frac{s-X^{(t+1)}}{(1-\beta) s}\mid X^{(t)}=s} \;\le\; \frac{\delta}{1-\beta}.
\end{align*}

Altogether, we obtain $\expect{\Delta_{t+1}(s)}\le
(\beta+1/(1-\beta))\delta \le ((\beta+1)/(1-\beta))\delta$.  From
Theorem~\ref{thm:pol-drift}, it now follows that
\begin{align*}
\expect{T\mid X^{(0)}=s_0} \;\ge\;
 \frac{1}{\delta}\cdot \frac{1-\beta}{1+\beta}\cdot \mathord{\ln}\mathord{\left( \frac{s_0}{\smin}\right)}.
\end{align*}
\end{proofof}

\end{appendix}

\end{document}